\documentclass[conference]{IEEEtran}
\usepackage{times}

\usepackage[numbers]{natbib}
\usepackage{multicol}
\usepackage{bm}
\usepackage{amsmath}
\usepackage{amsthm}
\usepackage{amssymb}
\usepackage{color}
\usepackage{subcaption}
\usepackage{mathtools}
\usepackage[bookmarks=true]{hyperref}
\usepackage{adjustbox}
\usepackage{array}
\usepackage{algorithm}
\usepackage[noend]{algpseudocode}

\newcolumntype{R}[2]{%
    >{\adjustbox{angle=#1,lap=\width-(#2)}\bgroup}%
    l%
    <{\egroup}%
}

\DeclareMathOperator*{\argmax}{arg\,max}
\DeclareMathOperator*{\argmin}{arg\,min}

\newtheorem{example}{Example}[section]

\newtheorem{theorem}{Theorem}
\newtheorem{lemma}{Lemma}

\begin{document}

\title{Bayesian Optimization for Polynomial Time Probabilistically Complete STL Trajectory Synthesis}

\author{Vince Kurtz and Hai Lin}



%

\maketitle

\begin{abstract}
In recent years, Signal Temporal Logic (STL) has gained traction as a practical and expressive means of encoding control objectives for robotic and cyber-physical systems. The state-of-the-art in STL trajectory synthesis is to formulate the problem as a Mixed Integer Linear Program (MILP). The MILP approach is sound and complete for bounded specifications, but such strong correctness guarantees come at the price of exponential complexity in the number of predicates and the time bound of the specification. In this work, we propose an alternative synthesis paradigm that relies on Bayesian optimization rather than mixed integer programming. This relaxes the completeness guarantee to probabilistic completeness, but is significantly more efficient: our approach scales polynomially in the STL time-bound and linearly in the number of predicates. We prove that our approach is sound and probabilistically complete, and demonstrate its scalability with a nontrivial example.  
\end{abstract}

\IEEEpeerreviewmaketitle

\section{Introduction and Related Work}

Signal Temporal Logic (STL) can be used to describe a wide range of control objectives in many application domains. STL formulas can specify the desired collective behavior of multi-agent systems, the task and motion plans of a robot, the movement of autonomous vehicles, and much more. STL is also highly interpretable: STL formulas are concise and can be easily understood by humans. 

Given an STL formula encoding a control objective, our primary goal is to design control inputs such that the resulting trajectory satisfies the STL formula. In the formal methods literature, this is known as the synthesis problem. 

In the context of control systems, the STL synthesis problem can be formulated as a trajectory optimization problem. Specifically, we can optimize over the ``robustness degree'', a scalar function which indicates the degree to which an STL specification is satisfied by a given trajectory. Unfortunately, the robustness degree is highly non-convex, especially for more complex STL specifications (which correspond to more interesting control objectives). 

The state-of-the-art is to formulate this optimization as a Mixed Integer Linear Programming (MILP) problem. This formulation was first proposed in \cite{raman2014model}, with slight variations on the encoding presented in a number of papers since then \cite{belta2019formal}. The MILP method is complete, since the MILP is solved to global optimality. However, the complexity is exponential in the time bound of the formula and the number of predicates \cite{belta2019formal}. This means that it is easy to make a given problem intractable, just by increasing the time bound of the specification. Furthermore, the MILP approach is not tractable for problems with many predicates, which restricts its applicability to relatively simple specifications. 

With this in mind, recent research has focused on alternative optimization paradigms for STL trajectory synthesis. Such approaches generally do not provide as strong of guarantees as the MILP approach, but they are often more effective in practice. 

One such method is to use gradient descent methods to optimize the robustness degree directly \cite{abbas2014functional,pant2018fly} . To do so, \citet{pant2017smooth} provide a smooth approximation of the gradient of the robustness degree. This method is faster and can handle nonlinear systems and nonlinear predicates, but since gradient descent can only find local minima, the approach is not complete. 

In a similar vein, \cite{dey2016fast} propose a sampling-based motion planner based on RRT* that is probabilistically complete, but can only handle a convex fragment of STL. This means that specifications that include disjunctions or nested temporal operators cannot be considered under this approach.

Another promising approach is to use Satisfiability Modulo Theories (SMT) to find a feasible solution, though not necessarily the optimal one \cite{shoukry2017linear,shoukry2016scalable}. This approach is intuitively attractive in the context of robotics, where satisfying the specification may be more desirable than finding a perfectly optimal trajectory. Early results for Linear Temporal Logic (LTL) specifications indicate good potential on a variety of interesting problems. SMT is a generalization of the NP-complete Boolean satisfiability checking problem, however, and avoiding the associated worst-case exponential complexity may be nontrivial. 

In this work, we propose a new optimization paradigm for STL trajectory optimization that is fast, probabilistically complete, and can handle the full STL semantics---not just a convex fragment. Specifically, we use Bayesian Optimization to synthesize a maximally satisfying trajectory. Bayesian Optimization is a natural choice for STL trajectory optimization, as it converges to the global optimum of a non-convex function, uses a minimal number of function evaluations, and has polynomial complexity. In addition to these desirable theoretical properties, we demonstrate that our approach outperforms the state-of-the-art MILP approach in practice on a nontrivial example.

The rest of this paper is organized as follows: Section \ref{sec:preliminaries} introduces some necessary preliminaries and gives a formal problem definition, Section \ref{sec:proposed approach} outlines our proposed approach and provides proofs of correctness, Section \ref{sec:example} provides an illustrative example of the scalability of our approach, and Section \ref{sec:conclusion} concludes the paper.

\section{Preliminaries}\label{sec:preliminaries}

\subsection{System}

Consider the deterministic discrete-time control system 
\begin{equation}\label{eq:system}
    \begin{aligned}
        & {x}_{t+1} = f({x}_t, {u}_t),  \\
        & {y}_t = g({x}_t, {u}_t), 
    \end{aligned}
\end{equation}
where ${x}\in \mathbb{R}^n$, ${u} \in \mathcal{U} \subset \mathbb{R}^m$, ${y} \in \mathbb{R}^p$, and the initial condition ${x}_0$ is given. Note that the control $u$ is constrained to lie in $\mathcal{U}$, which is a compact subset of $\mathbb{R}^m$. Our goal will be to find a control sequence $\mathbf{u} = [{u}_0, {u}_1, \dots]$ such that the resulting output signal $\mathbf{y} = [y_0, y_1, \dots y_N]$ satisfies a given STL specification. 

Note that we do not assume that the state space is bounded, though this is likely the case in many practical applications, as such restrictions can be encoded elegantly in the STL specification. 

\subsection{Signal Temporal Logic}

As with any formal language, we define STL by specifying its \textit{syntax} and \textit{semantics}. We review these breifly here. For a more in-depth treatment, we refer the reader to \cite{belta2019formal,donze2010robust,maler2004monitoring} and references therein. 

We define STL over the output signal $\mathbf{y} = [y_0, y_1, \dots y_N]$ of System (\ref{eq:system}). STL is defined recursively over atomic predicates in the form $\pi = (\mu(\mathbf{y}) \geq c)$. The function $\mu(\cdot)$ is often assumed to be linear, though we need not make this restriction in our proposed approach. The STL syntax is formally defined as follows:
\begin{gather*}
    \varphi := \pi \mid \lnot \varphi \mid \varphi_1 \land \varphi_2 \mid \varphi_1 \lor \varphi_2 \mid \varphi_1 \mathbf{U}_{[t_1,t_2]} \varphi_2
\end{gather*}
where $\pi$ is an atomic predicate as defined above and $\varphi, \varphi_1, \varphi_2$ are STL formulas. 

The semantics, or meaning, of an STL formula are also recursively defined. We denote a signal $\mathbf{y}$ which satisfies a specification $\varphi$ as $\mathbf{y} \vDash \varphi$. Similarly, if $[y_t, y_{t+1}, \dots]$ satisfies $\varphi$, we write $\mathbf{y} \vDash_t \varphi$.

\begin{itemize}
    \item $\mathbf{y} \vDash \varphi \iff \mathbf{y} \vDash_0 \varphi$
    \item $\mathbf{y} \vDash_t \pi \iff \mu([y_t,y_{t+1},\dots]) \geq c$
    \item $\mathbf{y} \vDash_t \lnot \varphi \iff \mathbf{y} \nvDash_t \varphi$
    \item $\mathbf{y} \vDash_t \varphi_1 \land \varphi_2 \iff \mathbf{y} \vDash_t \varphi_1 \land \mathbf{y} \vDash_t \varphi_2$
    \item $\mathbf{y} \vDash_t \varphi_1 \lor \varphi_2 \iff \mathbf{y} \vDash_t \varphi_1 \lor \mathbf{y} \vDash_t \varphi_2$
    \item $\mathbf{y} \vDash_t \varphi_1 \mathbf{U}_{[t_1, t_2]} \varphi_2 \iff \exists t^\prime \in [t_1, t_2] \text{ such that } \mathbf{y} \vDash_{t^\prime} \varphi_2 \text{ and } \forall t^{\prime\prime} \in [t_1, t^\prime], \mathbf{y} \vDash_{t^{\prime\prime}} \varphi_1$
\end{itemize}
In addition to the boolean operators (``and'', ``or'', ``not'') and the temporal operator ``until'' defined above, we can define temporal operators for ``eventually'' and ``always'':
\begin{itemize}
    \item $\mathbf{y} \vDash_t \mathbf{F}_{[t_1,t_2]} \varphi \iff \mathbf{y} \vDash_t True\mathbf{U}_{[t_1,t_2]}\varphi$
    \item $\mathbf{y} \vDash_t \mathbf{G}_{[t_1,t_2]} \varphi \iff \mathbf{y} \vDash_t \lnot\mathbf{F}_{[t_1,t_2]}\lnot\varphi$
\end{itemize}

In this work, we restrict our attention to bounded STL specifications, that is, those for which the bounds on temporal operators like ``until'', ``always'', and ``eventually'' are finite. 

\begin{example}
    Consider the specification 
    \begin{equation*}
        \varphi = \mathbf{G}_{[0,3]}\big( (\mathbf{y}>3)\land \lnot(\mathbf{y}>6)\big)
    \end{equation*}
    This specification essentially states that the signal $\mathbf{y}$ should remain between 3 and 6 for the first 4 timesteps.
    
    We can see that $\mathbf{y}_1 = [4,4,5,5,7] \vDash \varphi$, but $\mathbf{y}_2 = [4,5,6,7,7] \nvDash \varphi$.
\end{example}

STL semantics determine whether a given signal $\mathbf{y}$ satisfies a specification or not. But to perform synthesis, it is often more useful if we can tell \textit{how well} $\mathbf{y}$ satisfies the specification. The STL robust semantics (also called quantitative semantics) define such a real-valued function $\rho^\varphi(\mathbf{y})$ which is positive if and only if $\mathbf{y} \vDash \varphi$:

\begin{itemize}
    \item $\rho^\pi(\mathbf{y},t) = \mu([y_t, y_{t+1},\dots]) - c$
    \item $\rho^{\lnot\varphi}(\mathbf{y},t) = -\rho^\varphi(\mathbf{y},t)$
    \item $\rho^{\varphi_1 \land \varphi_2}(\mathbf{y},t) = \min(\rho^{\varphi_1}(\mathbf{y},t),\rho^{\varphi_2}(\mathbf{y},t))$
    \item $\rho^{\varphi_1 \lor \varphi_2}(\mathbf{y},t) = \max(\rho^{\varphi_1}(\mathbf{y},t),\rho^{\varphi_2}(\mathbf{y},t))$
    \item 
    $
    \begin{aligned}[t]\rho^{\varphi_1 \mathbf{U}_{[t_1,t_2]} \varphi_2}(\mathbf{y},t) = \max_{t^\prime \in [t_1,t_2]}(\min(&\rho^{\varphi_2}(\mathbf{y},t),\\
    & \min_{t^{\prime\prime}\in[t,t^\prime]}\rho^{\varphi_1}(\mathbf{y},t)))\end{aligned}$
    \item $\rho^\varphi(\mathbf{y}) = \rho^\varphi(\mathbf{y},0)$ 
\end{itemize}

\subsection{Bayesian Optimization}

Bayesian optimization is a black-box non-convex global optimization method. ``Black box'' means that we do not need specific information about the objective function, we simply need to be able to record its output for a given input. The basic idea is to maintain an approximation of the objective function, usually in the form of a Gaussian Process (GP), and optimize over the approximation.

A GP is a random process with the unique property that any finite subset is a multivariate gaussian random variable. We typically specify a GP $\mathbf{Y}$ indexed by $\mathbf{x}$ with a mean function $\mu(\mathbf{x})$ and a kernel function $k(\mathbf{x},\mathbf{x}')$, such that
\begin{equation}
    \mathbf{Y} \sim \mathcal{GP}(\mu(\mathbf{x}),k(\mathbf{x},\mathbf{x}')
\end{equation}
implies
\begin{equation}
    \mathbf{Y}_{\{\mathbf{x}_1,\dots,\mathbf{x}_N\}}
    \sim
    \mathcal{N}(\bm{\mu},\bm{K})
\end{equation}
where $[\bm{\mu}]_{\mathbf{x}_i} = \mu(\mathbf{x}_i)$ and $[\bm{K}]_{ij} = k(\mathbf{x}_i,\mathbf{x}_j)$.

GPs can be used for regression by calculating the conditional distribution of unobserved function values given the observations, under the assumption that observations are drawn from a GP prior with a given kernel.

Of course, not all functions can be exactly represented with a given kernel function. To describe how well a given function $f$ can be described by kernel function $k(\cdot,\cdot)$, we turn to the notion of a Reproducing Kernel Hilbert Space (RKHS) \cite{hofmann2008kernel}. The RKHS of kernel $k(\cdot,\cdot)$ is a Hilbert space $\mathcal{H}_k$ defined by an inner product $\langle \cdot, \cdot \rangle_k$ such that the \textit{reproducing property} holds:
\begin{equation*}
    \langle f, k(\mathbf{x},\cdot) \rangle_k = f(\mathbf{x}) ~~ \forall f \in \mathcal{H}_k.
\end{equation*}

This space induces the RKHS norm $\|f\|_k = \sqrt{\langle f,f \rangle_k}$. This norm is essentially a measure of the smoothness of the function $f$ with respect to the smoothness of the kernel function. 

While not all continuous functions are in $\mathcal{H}_k$ for a given kernel, certain kernels such as the popular Mat\'ern kernel \cite{rasmussen2003gaussian}, do have the ability to approximate continuous functions with arbitrary precision, as shown in the following lemma:

\begin{lemma}\label{lemma:approximation}
For any continuous function $f : \mathcal{X} \mapsto \mathbb{R}$, scalar constant $\epsilon > 0$, and compact subset $\mathcal{Z} \subset \mathcal{X}$, there exists some function $g : \mathcal{X} \mapsto \mathbb{R}$ such that $g \in \mathcal{H}_k$ for Mat\'ern kernel $k(\mathbf{x},\mathbf{x}^\prime)$ and
\begin{equation*}
    |f(\mathbf{x})-g(\mathbf{x})| \leq \epsilon
\end{equation*}
for all $\mathbf{x} \in \mathcal{Z}$.
\end{lemma}

This \textit{universal approximation} property is a well-known result in the Gaussian Process literature, and is responsible for much of the success of Gaussian process regression. For further details of this result, we refer the reader to \cite{rasmussen2003gaussian,micchelli2006universal}. 

In Bayesian Optimization, the GP approximation of the target function is updated at each step with a new sample. Each subsequent sample is taken according to an \textit{aquisition function}, which regulates the tradeoff between exploration (searching new areas of the input space) and exploitation (using what we already know to find the extremum). 

For a given input $\mathbf{x}$, the GP estimate is a normal distribution 
\begin{equation*}
    \hat{f}(\mathbf{x}) = \mathcal{N}\Big(\mu_{t-1}(\mathbf{x}),\sigma^2_{t-1}(\mathbf{x})\Big).
\end{equation*} 
In this work, we consider the Upper Confidence Bound (UCB) aquisition function\footnote{In the case of minimization, this acquisition function would be more properly referred to as the lower confidence bound. However, as UCB is the dominant notation in the literature, we will refer to it as the UCB.}, which regulates the exploitation-exploration tradeoff with parameter $\beta_t$:
\begin{equation*}
    \mathbf{x}_{t} = \argmax_{\mathbf{x}}\Big(\mu_{t-1}(\mathbf{x})+\sqrt{\beta_t}\sigma_{t-1}(\mathbf{x})\Big).
\end{equation*}

This case of Bayesian Optimization is known as GP-UCB, and has well studied convergence properties which we will draw on to prove probabilistic completeness of our approach. Such guarantees are typically expressed in terms of the maximum information gain from noisy observations $\bm{y} = f(\mathbf{x})+ \epsilon$, $\epsilon \sim \mathcal{N}(0,\sigma^2\mathbf{I})$,
\begin{equation*}
    \gamma_T := \max I(\bm{y};f(\mathbf{x})),
\end{equation*}
where $I(\bm{y};f(\mathbf{x}) = \frac{1}{2}\log|\mathbf{I}+\sigma^{-2}\bm{K}|$ expresses the information gain after $T$ rounds \cite{srinivas2012information}.

\subsection{Problem Formulation}

Our goal is to find a sequence of control inputs $\mathbf{u} = [u_0, u_1, \dots, u_N]$ such that the resulting output sequence $\mathbf{y} = [y_0, y_1, \dots, y_N]$ satisfies a given bounded STL specification $\varphi$ with a minimum robustness degree $\rho_{min} > 0$.

We formulate this goal as a nonconvex optimization problem as follows:
\begin{equation}\label{eq:problem}
\begin{aligned}
    \min_{\mathbf{u}} ~& J(\mathbf{u}) := -\rho^\varphi(\mathbf{y}) \\
    \text{s.t. } & x_{t+1} = f(x_t, u_t) \\
         & y_t = g(x_t, u_t).
\end{aligned}
\end{equation}
If $J(\mathbf{u}^*) \leq -\rho_{min}$, we conclude that the trajectory resulting from control sequence $\mathbf{u}^*$ satisfies $\varphi$ with at least robustness degree $\rho_{min}$.

\section{Proposed Approach}\label{sec:proposed approach}

\subsection{Overview}

Our proposed approach is outlined in Algorithm \ref{alg:overview}. The basic idea is to use GP-UCB to minimize the cost function $J(\mathbf{u})$. We first generate a set of initial guesses for $\mathbf{u}$ and update our estimate of the cost function, $\hat{J}(\cdot)$, accordingly. Then we alternate between choosing a new estimated $\mathbf{u}$ according to the UCB and updating $\hat{J}(\cdot)$ with the true cost function at $\mathbf{u}$. Note that the true value of $J(\mathbf{u})$ can be computed for a given initial condition by simulating System (\ref{eq:system}) forward in time. 

\begin{algorithm}
    \caption{STL Trajectory Synthesis with Bayesian Opt.}
    \label{alg:overview}
    \begin{algorithmic}
        \Procedure{BayesSTL}{$\rho_{min},\varphi,N$}
        \State $\hat{J} = \mathcal{GP}(0,k(\mathbf{u},\mathbf{u}^\prime))$
        \State $\{\mathbf{u}_0\} \gets $ initial guesses
        \State $\hat{J} \gets \text{update\_GP}(\{\mathbf{u}_{0}\},\{J(\mathbf{u}_{0})\})$
        \For{ $ i = [1..N]$ }
            \State $\mathbf{u}_{i} \gets \text{UCB}(\hat{J})$
            \If{ $J(\mathbf{u}_{i}) \leq - \rho_{min}$ }
                \Return $\mathbf{u}_{i}$ 
            \EndIf
            \State $\hat{J} \gets \text{update\_GP}(\mathbf{u}_{i},J(\mathbf{u}_{i}))$
        \EndFor
        \State \Return infeasible
        \EndProcedure
    \end{algorithmic}
\end{algorithm}

The algorithm terminates when we have found a control sequence $\mathbf{u}$ such that $J(\mathbf{u}) \leq - \rho_{min}$, indicating that the resulting trajectory robustly satisfies the specification $\varphi$, or the user-specified maximum number of iterations, $N$, is reached. 

Intuitively, Bayesian optimization is a natural choice for solving the optimization problem (\ref{eq:problem}), as we have little insight into the structure of the nonconvex cost function $J(\cdot)$ for an arbitrary specification $\varphi$, but we can evaluate $J(\mathbf{u})$ for a given control tape $\mathbf{u}$. 

Furthermore, Bayesian optimization is highly efficient for scenarios in which the objective function is noisy and difficult to evaluate. It is often used for problems like hyperparameter tuning in deep neural networks, where each function evaluation may take on the order of hours and the output could be stochastic. This suggests that our approach can be extended to non-deterministic systems with high-dimensional state spaces and complex, nonlinear dynamics. In such scenarios, simulating System (\ref{eq:system}) to calculate $J(\mathbf{u})$ may take some time, and the value of $J(\mathbf{u})$ may be stochastic. 

\subsection{Correctness}

Beyond being intuitively reasonable, our approach offers significant theoretical guarantees. While we cannot prove completeness as in the case of MILP optimization, we can guarantee soundness and probabilistic completeness for bounded STL specifications. Furthermore, while MILP optimization is exponential in the number of predicates and the time horizon of the STL formula $\varphi$, the computational complexity of our approach scales polynomially in the time horizon of the STL formula and linearly in the number of predicates. 

We formalize these notions as follows:

\begin{theorem}(Soundness). 
    Algorithm \ref{alg:overview} finds a control sequence $\mathbf{u}$ only if the resulting output trajectory $\mathbf{y}$ satisfies the given STL specification  $\varphi$. 
\end{theorem}
\begin{proof}
    Algorithm \ref{alg:overview} returns a control sequence $\mathbf{u}$ only if $J(\mathbf{u}) \leq -\rho_{min} < 0 \implies \rho^\varphi(\mathbf{y}) > 0 \implies \mathbf{y} \vDash \varphi$, and so the Theorem trivially holds. 
\end{proof}

To prove probabilistic completeness, we first note that GP-UCB converges to the true optimum of a function in the RKHS of its kernel with high probability, as shown in the following lemma:

\begin{lemma}\cite{srinivas2012information}.\label{lemma:convergence}
Let $\delta \in (0,1)$. Assume that the true objective function $f$ lies in the RKHS corresponding to the kernel $k(\bm{x},\bm{x}^\prime)$ and that the noise $\epsilon_t$ has zero mean conditioned on the history and is bounded by $\sigma$ almost surely. Assume that $\|f\|_k^2 \leq B$ and let $\beta_t = 2B+300\gamma_t\log^3(t/\delta)$. Running GP-UCB with parameter $\beta_t$ and prior $\mathcal{GP}(0,k(\bm{x},\bm{x}^\prime))$, we can bound the the error after sample $T\geq1$ as follows:
\begin{equation*}
    P\Big(f(\bm{x}^*)-f(\bm{x}_T) \leq \sqrt{C_1T\beta_T\gamma_T}\Big) \geq 1-\delta
\end{equation*}
where $C_1 = 8\log(1+\sigma^{-2})$.
\end{lemma}

Furthermore, note that any continuous function---including our cost function $J(\cdot)$---can be approximated to arbitrary precision by a function in the RKHS of a Mat\'ern kernel in the sense of Lemma \ref{lemma:approximation}.

We can now state the probabilistic completeness of our algorithm as follows:

\begin{theorem}\label{theorem:prob_complete}(Probabilistic Completeness). For sufficient $N$, Algorithm \ref{alg:overview} will find a control sequence $\mathbf{u}^*$ such that the resulting output trajectory $\mathbf{y}$ satisfies the specification $\varphi$ with robustness degree $\rho_{min}$, provided such a control sequence exists, with high probability. Specifically, for a given iteration $i \in [1,N]$, the error between the $i^{th}$ estimate $J(\mathbf{u}_i)$ and the optimal cost $J(\mathbf{u}^*)$ is bounded as follows:
\begin{equation}\label{eq:error_bound}
    P\Big(J(\mathbf{u}^*)-J(\mathbf{u}_i) \leq \sqrt{C_1i\beta_i\gamma_i} + 2\epsilon \Big) \geq 1-\delta
\end{equation}
for $\delta \in (0,1)$, $\epsilon > 0$, and $C_1$, $\beta_i$ as defined in Lemma \ref{lemma:convergence}.
\end{theorem}

\begin{proof}
    First, note that $J(\cdot)$ can be approximated with arbitrary precision by a function $\hat{J}(\cdot)$ in the RKHS of Mat\'ern kernel $k(\cdot,\cdot)$ in the sense of Lemma \ref{lemma:approximation}, as the control sequence $\mathbf{u}$ lies within the closed set 
    \begin{equation*}
        \mathcal{Z} = \big\{\mathbf{u}=[u_0,u_1,\dots,u_N] \mid u_i \in \mathcal{U} ~ \forall i \in [0,N]\big\}.
    \end{equation*}
    
    For small $\epsilon$, we can bound the error in $\hat{J}(\cdot)$ at the $i^{th}$ step by 
    \begin{equation*}
        P\Big(\hat{J}(\mathbf{u}^*)-\hat{J}(\mathbf{u}_i) \leq \sqrt{C_1i\beta_i\gamma_i} \Big) \geq 1-\delta
    \end{equation*}
    by Lemma \ref{lemma:convergence}.
    
    It then follows from the fact that $|J(\mathbf{u})-\hat{J}(\mathbf{u})| \leq \epsilon$ for $\mathbf{u} \in \mathcal{Z}$ that
    \begin{equation*}
        P\Big(J(\mathbf{u}^*)-J(\mathbf{u}_i) \leq \sqrt{C_1i\beta_i\gamma_i} + 2\epsilon \Big) \geq 1-\delta
    \end{equation*}
    for small $\epsilon$.

    Finally, we say that $\mathbf{y} \vDash \varphi$ with robustness degree $\rho_{min}$ if and only if $J(\mathbf{u}) \leq -\rho_{min}$. Since our algorithm converges to the true optimum $\mathbf{u}^* = \argmin_{\mathbf{u}}J(\mathbf{u})$ with high probability in the sense of (\ref{eq:error_bound}), the Theorem holds.  
\end{proof}

In short, this theorem states that our algorithm will converge to the globally optimal $\mathbf{u}^*$ as the number of iterations approaches infinity, and thus will eventually find a satisfying solution if one exists. 

In practical terms, Equation (\ref{eq:error_bound}) provides useful insight into how to pick the iteration upper bound $N$. Larger $N$ leads to an tighter bound on the final error between the optimal and the estimated cost function, and this bound can be computed by setting $i=N$ in Equation (\ref{eq:error_bound}).

Of course, there is a tradeoff between choosing a higher $N$, thus obtaining a more complete result, and the ensuing computational expense of performing more iterations of Algorithm \ref{alg:overview}. We now give a more thorough characterization of this tradeoff by quantifying the computational complexity of our approach, as shown in the following Theorem:

\begin{theorem}(Computational Complexity).
    The worst-case computational complexity of Algorithm \ref{alg:overview} is $\mathcal{O}(|\Pi|N(T+N)^3)$, where $|\Pi|$ is the number of predicates in STL formula $\varphi$, $T$ is the time bound of $\varphi$, and $N$ is the user-specified maximum number of iterations of Algorithm \ref{alg:overview}.
\end{theorem}
\begin{proof}
    Algorithm \ref{alg:overview} requires inverting a $m(T+i)\times m(T+i)$ covariance matrix to update the Gaussian Process estimate $\hat{J}(\cdot)$ at each step $1 \leq i \leq N$. This is an $\mathcal{O}((T+i)^3)$ process which is repeated $N$ times, and thus we can bound the complexity with $\mathcal{O}(N(T+N)^3)$. 
    
    In terms of the number of predicates $|\Pi|$, note that the size of this covariance matrix is independent of the structure of $J(\cdot)$, and thus of $|\Pi|$. The only way the structure of $J(\cdot)$ enters the picture is in evaluating $J(\mathbf{u})$. The complexity of each evaluation depends linearly on $|\Pi|$, a fact that can be easily seen from the STL robust semantics. 
\end{proof}

\subsection{Initialization}\label{sec:initialization}

While our approach is probabilistically complete and polynomial in complexity, the practical run-time of Algorithm \ref{alg:overview} depends heavily on the choice of initial control candidates $\{\mathbf{u}_0\}$. In this section, we discuss the choice of such initial candidates and describe how heuristic optimization methods like Differential Evolution \cite{storn1997differential} can significantly improve the performance of our method in practice. 

Since Algorithm \ref{alg:overview} relies on Gaussian Process regression, the time complexity is cubic in the number of samples of the cost function $J(\mathbf{u})$. Furthermore, the space of all possible control sequences is potentially huge for large time bounds. Finding a suitable optimum over such a large space inevitably requires many samples of $J(\mathbf{u})$ unless we have a good initial guess, or set of guesses, about where the optimum may be. 

\begin{figure*}
    \begin{subfigure}{0.49\linewidth}
        \includegraphics[width=\linewidth]{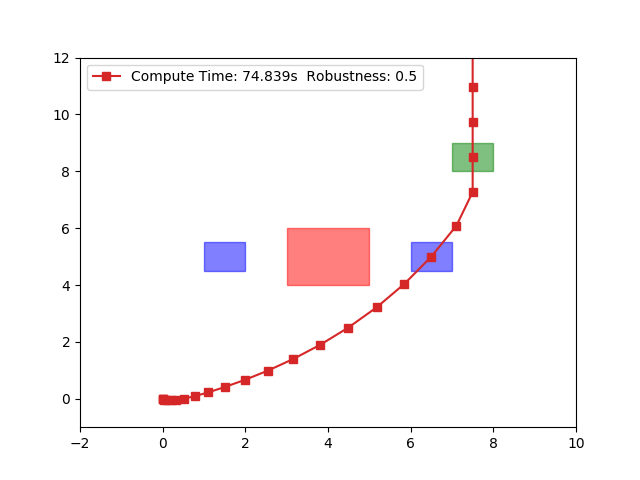}
        \caption{Mixed Integer Programming (state-of-the-art)}
        \label{fig:milp_trajectory}
    \end{subfigure}
    \begin{subfigure}{0.49\linewidth}
        \includegraphics[width=\linewidth]{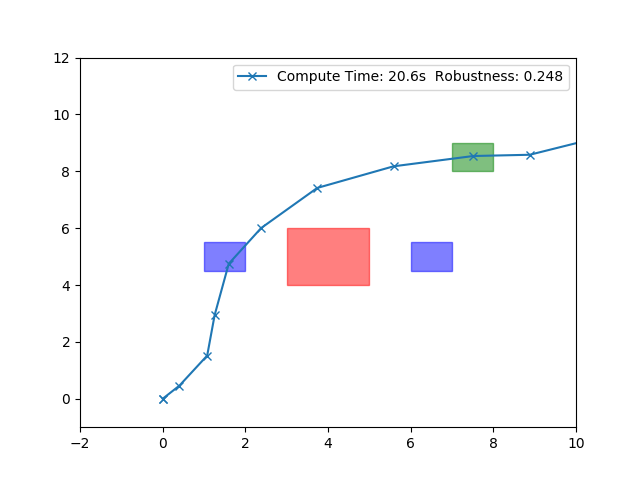}
        \caption{Bayesian Optimization with DE initialization (ours)}
        \label{fig:our_trajectory}
    \end{subfigure}
    \begin{subfigure}{0.49\linewidth}
        \includegraphics[width=\linewidth]{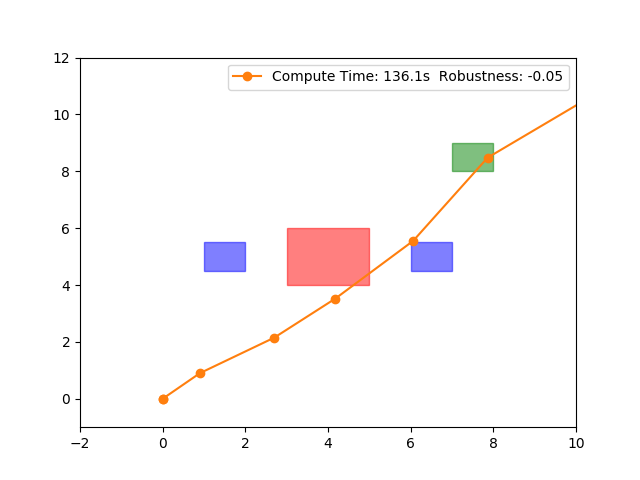}
        \caption{Bayesian Optimization Only}
        \label{fig:bayes_trajectory}
    \end{subfigure}
    \begin{subfigure}{0.49\linewidth}
        \includegraphics[width=\linewidth]{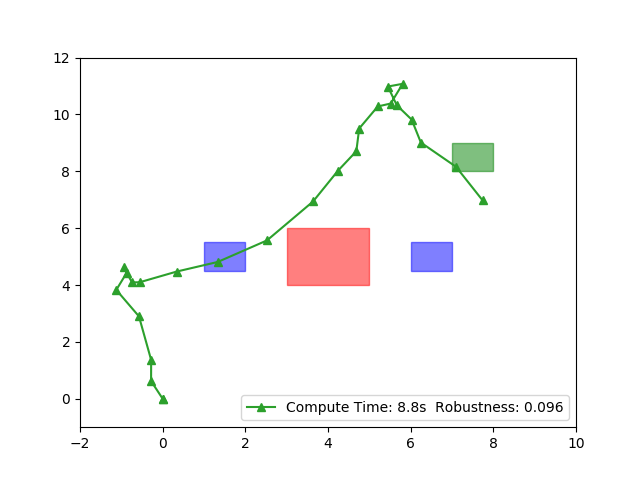}
        \caption{Differential Evolution Only}
        \label{fig:de_trajectory}
    \end{subfigure}
    \caption{Examples of trajectories generated by different optimization approaches to satisfy specification \ref{eq:specification}. A robot starting at $(0,0)$ must visit the green region and one of the blue regions while avoiding the central red region. Our approach (\ref{fig:our_trajectory}) provides a superior balance between computational complexity and robustness. }
    \label{fig:trajectories}
\end{figure*}

We propose using Differential Evolution \cite{storn1997differential} to find such a set of initial guesses $\{\mathbf{u}_0\}$. Differential Evolution (DE) is a genetic-type heuristic algorithm for maximizing an unknown function. The basic idea is that a set of candidate solutions is updated through simple mutation and combination operations at each iteration, gradually converging towards the global optimum. While DE does not provide theoretical convergence guarantees, it has been shown to be effective on a wide variety of difficult and high-dimensional problems.

Furthermore, combining candidate solutions to generate new solutions makes intuitive sense in the context of STL trajectory optimization: one part of a certain trajectory may be good for achieving a certain sub-task, while another trajectory may have components that advance another part of the specification. DE is fast, and several initial guesses can be easily computed in parallel. Furthermore, we have found that even running DE for a few ($<10$) iterations can provide reasonable initial guesses. The lack of convergence guarantees is not a significant concern in this initialization phase, as the subsequent Bayesian optimization is guaranteed to converge with high probability, as demonstrated by Theorem \ref{theorem:prob_complete}.   

\section{Example}\label{sec:example}

\begin{figure}
    \centering
    \includegraphics[width=\linewidth]{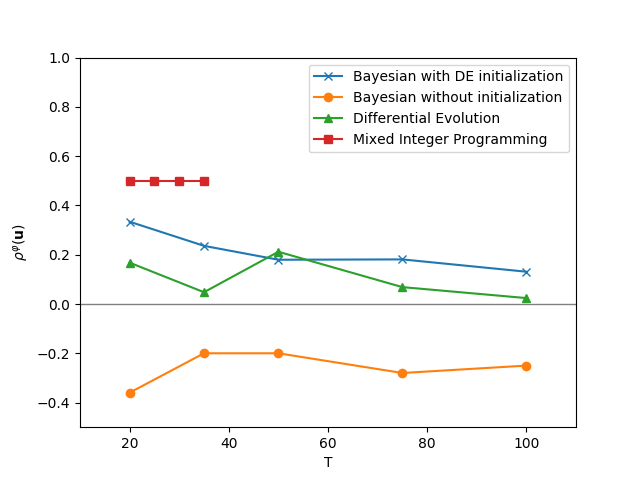}
    \caption{Robustness degree vs STL time bound.}
    \label{fig:robustness}
\end{figure}

\begin{figure}
    \centering
    \includegraphics[width=\linewidth]{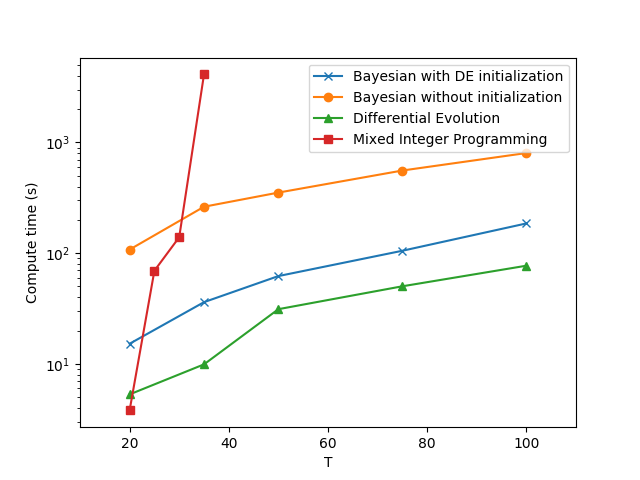}
    \caption{Computation time vs STL time bound (log scale).}
    \label{fig:computation}
\end{figure}

In this section, we present an example of using our approach to synthesize trajectories that satisfy an STL specification. 

Consider a robot with double integrator dynamics moving in a 2D workspace. The system is given as follows:
\begin{equation}
    \begin{aligned}
        x_{t+1} & =
        \begin{bmatrix}
            1 & 1 & 0 & 0 \\
            0 & 1 & 0 & 0 \\
            0 & 0 & 1 & 1 \\
            0 & 0 & 0 & 1 \\
        \end{bmatrix}
        x_t +
        \begin{bmatrix}
            0 & 0 \\
            1 & 0 \\
            0 & 0 \\
            0 & 1 \\
        \end{bmatrix}
        u_t \\
    y_t & = \begin{bmatrix}
        x^1_t &
        x^3_t &
        u^1_t &
        u^2_t
    \end{bmatrix}^\intercal
    \end{aligned}
\end{equation}
where $x \in \mathbb{R}^4$ represents the position and velocity of the robot, and $u \in \mathbb{R}^2$ represents a velocity control input. 

The STL formula $\varphi$ specifies that the robot should visit one of two waypoint regions and a goal region within $T$ steps, always avoid an obstacle, and keep the control inputs bounded in $[u_{min},u_{max}]$:
\begin{multline}\label{eq:specification}
    \varphi = \mathbf{F}_{[0,T]}(\text{waypoint1}\lor \text{waypoint2})\land\mathbf{F}_{[0,T]}(\text{goal}) \\
    \land 
    \mathbf{G}_{[0,T]}(\lnot\text{obstacle})
    \land \mathbf{G}_{[0,T]}(\text{control bounded})
\end{multline}
This specification is displayed in Figure \ref{fig:trajectories}. The waypoint regions are shown in blue and the goal region is shown in green, while the central obstacle region is shown in red.

Note that specification (\ref{eq:specification}) is non-convex, as it contains a disjuction. This means that related optimization approaches like gradient-descent \cite{abbas2014functional} and RRT* \cite{dey2016fast} are not guaranteed to find a solution. 

We tested four optimization techniques on this problem: mixed integer programming, our approach, differential evolution only, and Bayesian optimization without DE initialization. Example results for $T=25$ are shown in Figure \ref{fig:trajectories}. We implemented Algorithm \ref{alg:overview} in Python using scipy \cite{scipy} and skopt \cite{scikit-optimize}, and adopted the MILP implementation of \cite{sadraddini2015robust} for comparison. All experiments were performed on a laptop with an Intel i7 processor and 32GB RAM. 

MILP (Figure \ref{fig:milp_trajectory}) finds a trajectory that achieves the optimal robustness degree of $0.5$. Qualitatively, the trajectory is smooth and reasonable, avoiding the obstacle by a wide margin and passing through the center of the waypoint and goal regions. But this approach takes a long time, over a minute, to compute a satisfying trajectory.

At another extreme, differential evolution (Figure \ref{fig:de_trajectory}) finds a satisfying run very quickly (under 10 seconds), but the resulting robustness degree ($0.096$) is low. Qualitatively, we can see that this trajectory is clearly suboptimal: the trajectory is erratic and barely clips the corner of the goal region, though it does satisfy the specification. 

Using Bayesian optimization only (Figure \ref{fig:bayes_trajectory}) demonstrates the slow convergence discussed in Section \ref{sec:initialization}. In this example, the maximum number of iterations $N=100$ was exceeded before a satisfying trajectory could be found, resulting in a robustness degree of $-0.05$. We can see qualitatively, however, that this approach is on the way towards a smooth satisfying trajectory.

Finally, our approach (Figure \ref{fig:our_trajectory}) demonstrates the best balance between robustness and computational speed. It takes under 30 seconds to find a satisfying trajectory, slower only than differential evolution alone. At the same time, the final trajectory is much better quality, achieving a robustness degree of $0.248$. Qualitatively, the resulting trajectory is similar to the optimal one generated by MILP---smooth, far from the obstacle, passing near the center of the waypoint and goal---but was computed in less than a third of the time. 

By varying $T$, we can investigate the effect of larger time bounds on the optimization efficiency. Note that $T$ can also be thought of as a ``stand-in'' for a parameter that varies the complexity of the specification, as the MILP approach scales exponentially with both $T$ and the number of predicates in $\varphi$.

The results shown in Figures \ref{fig:robustness} and \ref{fig:computation} further illustrate that our approach provides a good balance between speed and optimality. The MILP approach always achives the maximum robustness degree ($0.5$), but its exponential complexity renders specifications with $T > 35$ infeasible. Using Differential Evolution directly is the fastest approach, but there is no completeness guarantee. Bayesian optimization without initialization (i.e. with a random initial guess), meanwhile, does not converge to a satisfying solution in the iteration bound provided ($N=100$). Our approach, on the other hand, finds a solution quickly and with provable convergence. While our approach, as might be expected, is slightly slower than differential evolution alone, this added computation time results in greater robustness.

\section{Conclusion}\label{sec:conclusion}

We presented a new approach for STL trajectory synthesis based on Bayesian Optimization. We proved that this approach is sound, probabilistically complete, and highly scalable. Specifically, our approach scales polynomially in the STL time bound and linearly in the number of predicates. On a practical level, we discussed the importance of providing a good initial guess for Bayesian Optimization. To this end, we proposed the use of Differential Evolution, a heuristic global optimization algorithm, to generate a useful set of initial conditions. We then showed in a practical example how our approach achieves superior performance to differential evolution alone, Bayesian Optimization without initialization, and the state-of-the-art MILP encoding. 

Future work will focus on developing feedback control strategies, extending these results to uncertain systems and unbounded specifications, and exploring more scalable Bayesian Optimization frameworks such as neural networks \cite{snoek2015scalable}.  

\bibliographystyle{plainnat}
\bibliography{references}

\end{document}